\documentclass[twocolumn]{article} 

\usepackage{preprint}

\usepackage{amsmath, amsthm, amssymb, amsfonts}
\usepackage{mathtools}

\usepackage[numbers,square]{natbib}
\usepackage[utf8]{inputenc}	
\usepackage[T1]{fontenc}	

\usepackage{graphicx}
\usepackage{multirow}
\usepackage{subcaption}
\usepackage{enumitem}
\usepackage[dvipsnames]{xcolor}
\usepackage{tikz}
\usepackage{siunitx}
\usepackage{algorithm}
\usepackage{algpseudocode}

\newtheorem{proposition}{Proposition}

\DeclareMathOperator*{\argmin}{arg\,min} 

\usepackage[colorlinks = true,
            linkcolor = purple,
            urlcolor  = blue,
            citecolor = cyan,
            anchorcolor = black,
            unicode=true, psdextra]{hyperref}	
\usepackage{booktabs} 		
\usepackage{float}			

\definecolor{lime}{HTML}{A6CE39}
\DeclareRobustCommand{\orcidicon}{
	\begin{tikzpicture}
	\draw[lime, fill=lime] (0,0)
	circle [radius=0.16]
	node[white] {{\fontfamily{qag}\selectfont \tiny ID}};
	\draw[white, fill=white] (-0.0625,0.095)
	circle [radius=0.007];
	\end{tikzpicture}
	\hspace{-2mm}
}
\foreach \x in {A, ..., Z}{\expandafter\xdef\csname orcid\x\endcsname{\noexpand\href{https://orcid.org/\csname orcidauthor\x\endcsname}
			{\noexpand\orcidicon}}
}

\usepackage{authblk}

\title{Weakly Convex Ridge Regularization for 3D\\
Non-Cartesian MRI Reconstruction\protect}

\author[1]{German Shâma Wache\orcidA{}}
\author[2,3]{Chaithya G R\orcidB{}}
\author[2,3,4]{Asma Tanabene}
\author[1,$\dagger$]{Sebastian Neumayer\orcidD{}}

\affil[1]{Faculty of Mathematics, Chemnitz University of Technology, Germany}
\affil[2]{NeuroSpin, Fr\'ed\'eric Joliot Institute for Life Sciences, CEA Paris-Saclay, France}
\affil[3]{MIND, Inria Saclay Centre, France}
\affil[4]{Siemens Healthineers, France}

\newcommand{\capFigOne}{Illustration of WCRR: The (rotated) inputs are processed by a filter bank $\{\mathbf W_1, \ldots, \mathbf W_J\}$. The extracted features are penalized based on the potentials $\{\psi_1, \ldots, \psi_J\}$. Usually, the filters extract high-frequency information.}

\newcommand{\capFigTwo}{Overview of the training and reconstruction pipeline. The WCRR is trained on a Gaussian denoising task, and then used for MRI reconstruction. This involves only tuning of the scalar hyperparameters $\sigma$ and $\lambda$ of the reconstruction model on a (small) validation set. The coil sensitivity maps are estimated using the ESPIRiT algorithm \cite{uecker2014espirit} on the central $24\times 24$ k-space data (simulating ACS acquisition).}

\newcommand{\capFigThree}{(A) 3D GoLF-SPARKLING trajectory with GRAPPA acceleration. The green portion highlights the central Cartesian readouts, and the blue one depicts the non-Cartesian SPARKLING parts. Slices with $k_x = 0$, $k_z = 0$, and $k_y = 0$ are given in (B), (C), and (D), respectively.}

\newcommand{\capTabOne}{Quantitative comparison between the rotation-invariant and the non-rotation-invariant versions of WCRR. Mean and standard deviation for masked PSNR (dB) and masked SSIM are given over the 12-coil testing set. Runtime for a single reconstruction and parameter count are also provided. In each column, the best value is highlighted in bold, and the second-best is underlined.}

\newcommand{\capTabTwo}{Mean and standard deviation of masked PSNR (dB) and masked SSIM over the 12-coil and 32-coil (unseen generalization task) test data for each method. Runtime for a single reconstruction, parameter count and training time are also provided when applicable. In each column, the best value is highlighted in bold, and the second-best is underlined.}

\newcommand{\capFigFour}{WCRR convergence curves for a representative volume. \emph{(Left)} Relative error between consecutive iterates (tolerance). \emph{(Middle)} Energy functional. \emph{(Right)} Masked PSNR. The dashed vertical line highlights where Algorithm~\ref{nmapg} terminates with tolerance \num{5e-3}.}

\newcommand{\capFigFive}{Cross-section (5th slice before the mid-slice) of a 12-coil representative volume along the \emph{sagittal}, \emph{coronal}, and \emph{transversal} axes. The highlighted region is shown magnified in the corresponding inset. The shorthands m-psnr and m-ssim refer to the masked metrics. Additionally, m-nmse refers to the masked normalized mean squared error.} 

\newcommand{\capFigSix}{Visualization of the learned rotation-invariant WCRR prior. The top rows show the learned 3D filters, where only weights with magnitude larger than \(15\%\) of the maximum absolute filter weight are displayed for clarity. The bottom plot shows the shared potential \(\phi_\beta\), with individual filter potentials obtained by rescaling this shared potential with the learned coefficients \(\alpha_i\) while preserving the same shape.
The localized positive and negative filter weights and the symmetric potential support the interpretation of WCRR as a learned sparsity prior.}

\newcommand{\capFigSeven}{Mid-plane reconstructions for prospective GS-SPARKLING accelerated acquisitions with $2\times 2$ GRAPPA (1 min) with (B) SENSE, (C) $\ell_1$-wavelets, (D) TV, (E) NC-PDNet, (F) DPIR, (G) Score-DiffMBIR and (H) WCRR. A reference volume (A) from a complete scan is also shown. We showcase \emph{sagittal}, \emph{coronal}, and \emph{transversal} views along with zoomed insets. The arrows highlight representative local differences between reconstructions.}

\newcommand{\capFigEight}{Mid-plane reconstructions for prospective CAIPIRINHA $3\times 2$ accelerated acquisitions (1 min) with (B) SENSE, (C) $\ell_1$-wavelets, (D) TV, (E) NC-PDNet, (F) DPIR, (G) Score-DiffMBIR and (H) WCRR. A reference volume (A) from a complete scan is also shown. We showcase \emph{sagittal}, \emph{coronal}, and \emph{transversal} views along with zoomed insets. The arrows highlight representative local differences between reconstructions.}

\begin{document}

\maketitle
\let\thefootnote\relax\footnotetext{\texttt{$\prescript{\dagger}{}{}$correspondence:sebastian.neumayer\\@math.tu-chemnitz.de}}


\begin{abstract}
While highly accelerated non-Cartesian acquisition protocols significantly reduce scan time, they often entail long reconstruction delays for 3D data.
Deep learning–based reconstruction methods can alleviate this, but typically rely on complex, highly parametrized architectures.
As an alternative, we introduce a rotation-invariant weakly convex ridge regularizer (WCRR). The resulting variational reconstruction approach is benchmarked against state-of-the-art methods on retrospectively simulated data and out-of-distribution prospective GoLF-SPARKLING and CAIPIRINHA acquisitions. Our approach achieves competitive reconstruction quality while using orders of magnitude fewer parameters and retaining the interpretability of a sparsity-promoting variational prior. In summary, WCRR provides a compact and interpretable learned variational prior for 3D non-Cartesian MRI reconstruction.


\end{abstract}





\section{Introduction}
\label{sec:introduction}
Magnetic resonance imaging (MRI) is a powerful technique enabling non-invasive visualization of anatomy and physiology with high soft-tissue contrast, which is key for the diagnosis of many medical conditions.
The corresponding measurements are acquired sequentially, which is a time-consuming process with low throughput, high susceptibility to motion artifacts, and potential for patient discomfort.
Parallel imaging techniques like SENSE \cite{pruessmann1999sense} and GRAPPA \cite{griswold2002generalized} leverage spatial redundancy from multi-coil receiver arrays to accelerate acquisition by uniform Cartesian undersampling of the $k$-space measurements.
Even higher acceleration can be achieved through variable density sampling (VDS) combined with compressed sensing (CS) techniques that exploit sparsity of the imaged object under an appropriate transform \cite{lustig2007sparse,liu2008sparsesense,Chauffert_SIAM2014,boyer2017compressed}.
Efficient VDS implementations require non-Cartesian sampling of k-space, such as spiral trajectories \cite{meyer1992fast,law2009interleaved}, twisting radial lines \cite{jackson1992twisting} and rosette trajectories \cite{noll1997multishot}, which have actually been proposed long before the advent of CS to provide flexible k-space coverage. Recent trajectory-design work continues to exploit this flexibility in highly accelerated 3D MRI. For example, GoLF-SPARKLING combines GRAPPA-based parallel imaging in the central k-space region with variable-density non-Cartesian sampling in the periphery to enable approximately one-minute, 1~mm isotropic whole-brain MPRAGE acquisitions~\cite{radhakrishna2026sparkling}.

Mathematically, the data acquisition is modeled with the Fourier transform  \cite{PottsSteidlTasche2001,fessler2003nonuniform}.
For each receiver coil $c \in \{1,\dots,C\}$ with sensitivity map $\mathbf S_c$ (which is encoded into a diagonal matrix), the associated k-space measurement $\mathbf y_c \in \mathbb{C}^{M}$ obeys
\begin{equation}
    \mathbf y_c = \mathcal{F}_\Omega \mathbf S_c \mathbf x + \mathbf n_c, 
\end{equation}
with $\mathbf x \in \mathbb{C}^{N}$ being the imaged object, $\mathcal{F}_\Omega$ being the Fourier transform at sample locations $\Omega = \{\mathbf k_m \}_{m=1}^{M}$ with
\begin{equation}
\big( \mathcal{F}_{\Omega} \mathbf x \big)_m 
= \sum_{n=1}^{N} \mathbf x_n e^{-i 2\pi \langle \mathbf k_m, \mathbf r_n \rangle}, 
\qquad m = 1, \ldots, M,
\label{nudft}
\end{equation}
and $\mathbf n = [\mathbf n_1, \dots, \mathbf n_C]^\top$ being additive white Gaussian noise that captures measurement errors.
In practice, the sensitivities $\mathbf S_c$ are unknown and need to be estimated from the data $\mathbf{y}_c$ through methods like ESPIRiT \cite{uecker2014espirit}. 
Stacking coils yields the linear forward model
\begin{equation}
\label{eq:model}
\mathbf y = \mathbf{A} \mathbf x + \mathbf n,
\end{equation}
with
\begin{equation}
\mathbf y = \begin{bmatrix}
\mathbf y_1\\
\vdots\\
\mathbf y_C
\end{bmatrix} \quad \text{and} \quad
\mathbf{A} = \begin{bmatrix}
\mathcal{F}_\Omega \mathbf S_1\\
\vdots\\
\mathcal{F}_\Omega \mathbf S_C
\end{bmatrix}.
\end{equation}

As we typically only have few samples $\mathbf k_1, \ldots, \mathbf k_M$ according to a non-uniform density, recovering $\mathbf x$ from \eqref{eq:model} is an ill-posed inverse problem with high sensitivity to the noise $\mathbf n$.
As a remedy, we pursue a theoretically grounded variational reconstruction approach \cite{scherzer2009variational}, where we minimize an objective consisting of a data-fidelity term and a regularization term $\mathcal{R}$ (promoting desired properties of $\mathbf x$), namely
\begin{equation}
\label{optimization}
\widehat{\mathbf x} = \argmin_{\mathbf x \in \mathbb{C}^{N}}\frac{1}{2} \left \| \mathbf A \mathbf x - \mathbf y\right \|_2^2 + \lambda \mathcal{R}(\mathbf x), 
\end{equation}
where $\lambda > 0$ balances the terms.
A typical CS-based regularizer for \eqref{optimization} is $\mathcal{R}(\mathbf{x}) = \|\mathbf{\Psi x}\|_1$, which promotes sparsity under a transform $\mathbf{\Psi}$ such as the gradient (leading to the total variation \cite{Rudin1992ROF}) or Wavelets \cite{Donoho1995SoftThreshold,candes2008introduction}.
To improve the reconstruction quality, several works aim at learning regularizers for \eqref{optimization}; see \cite{LearnedRegularizers} for an overview in 2D.
One example is the learnable fields-of-experts regularizer \cite{RotBla2009}, which was re-popularized as weakly convex ridge regularizer (WCRR) in \cite{goujon2024learning} and our starting point.

\paragraph{Related work}
In recent years, deep-learning-based approaches have become the state-of-the-art for solving inverse problems, see for example the reviews \cite{AMOS2019,OngJalBar2020,habring2024neural}.
However, several concerns regarding the stability and trustworthiness of these methods remain \cite{gottschling2020troublesome}.
They can be broadly grouped according to how they interact with the acquisition model $\mathbf{A}$.

A first line of work post-processes an initial reconstruction, such as a zero-filled or adjoint reconstruction, using a CNN trained to remove artifacts and noise~\cite{yang2018dagan,mardani2018gan,zhu2018automap}.
Plug-and-play methods extend this idea by inserting a learned denoiser, such as a DRUNet, as an implicit prior within an iterative reconstruction algorithm~\cite{venkatakrishnan2013plug,romano2017little,hurault2021gradient,ZhaLiZuo2022}.
Closely related are unrolled models~\cite{hammernik2018learning,adler2018learned,aggarwal2018modl,schlemper2019nonuniform}, which interleave data-consistency steps with learned update blocks.
In the non-Cartesian setting, NC-PDNet introduced a density-compensated primal-dual network~\cite{ramzi2022nc}, a direction recently extended to the 3D multi-coil case~\cite{tanabene2024benchmarking}.
Building on supervised unrolling, recent advances incorporate dual-domain self-supervision, where measured k-space is partitioned to provide supervision in both the k-space and image domains~\cite{zhou2023dual}, as well as sampling-pattern optimization, where non-Cartesian k-space locations and the reconstruction network are jointly learned~\cite{alkan2025autosamp}.
Generative and diffusion priors have recently attracted growing attention: Score-MRI couples score-based diffusion models with data-consistency updates~\cite{chung2022score}, while DiffMBIR deploys pretrained 2D diffusion priors within a 3D setting~\cite{chung2023solving}.
A systematic review of diffusion-based models for MRI reconstruction can be found in ~\cite{habijan2026diffusion}.
By contrast, the learning of regularizers for the variational approach \eqref{optimization} remain largely confined to the 2D Cartesian regime~\cite{KobEff2020,ZacKnoPoc2023,ZouLiuWoh2023,goujon2023neural}.
A 2D non-Cartesian setting was studied in~\cite{kofler2020neural}, and an extension to 3D Cartesian MRI was explored in~\cite{fuin2020multi}.

\paragraph{Contribution}
Learned regularizers have, to the best of our knowledge, not yet been specifically investigated for 3D non-Cartesian MRI reconstruction.
To close this gap, we extend and benchmark WCRR specifically for this setting.
Our contributions are threefold:

\begin{enumerate}[label=\roman*.]
    \item \textbf{A principled, scalable prior for 3D MRI data}: We adapt WCRR to complex-valued 3D inputs, while preserving its interpretability and optimization guarantees.
    
    \item \textbf{Rotation invariance}: We introduce a rotation-invariant formulation to decrease the model size and improve data efficiency.  

    \item \textbf{A comprehensive MRI evaluation}: We benchmark WCRR reconstruction against parallel imaging, compressed sensing, plug-and-play, image-domain postprocessing, unrolled networks, and diffusion models.
    This includes both retrospectively simulated acquisitions and prospective real-world acquisitions.

\end{enumerate}

\section{Methods}

Throughout, we interpret complex-valued image volumes as 2-channel (real and imaginary) real-valued ones.
First, we detail the architecture of the regularizer $\mathcal R$ that we deploy within the reconstruction model \eqref{optimization}.
Then, we describe its efficient training based on a denoising task. Finally, we give details on the experimental setup.

\subsection{Rotation-invariant fields-of-experts}
We restrict ourselves to the fields-of-experts model and introduce the rotation-invariant formulation
\begin{equation}
\label{reg}
\mathcal{R}(\mathbf x) = \left| \mathcal{G} \right|^{-1} \sum_{\mathrm R \in \mathcal{G}}\sum_{j=1}^{J} \langle \mathbf{1}, \psi_j(\mathbf W_j \mathrm R \mathbf x)\rangle,
\end{equation}
where $\mathcal{G}$ is a set of rotations, $\mathbf W=[\mathbf W_1, \ldots, \mathbf W_J]^\top$ a stack of convolution operators acting on the rotated versions of $\mathbf x \in \mathbb{R}^{2N}$, and the potentials $\psi=[\psi_1, \ldots,\psi_J]^\top$ with $\psi_j \colon \mathbb R \to \mathbb R$ are applied \emph{component-wise}.
The following proposition is the analog of \cite[Prop.\ 3.2]{goujon2024learning}.

\begin{proposition}
\label{prop_weak_conv}
If the $\psi_j$ are $1$-weakly convex and $\|\mathbf W\| = 1$, then $\mathcal{R}$ in \eqref{reg} is also $1$-weakly convex.
\end{proposition}
\begin{proof}
A function $g \colon \mathbb{R}^m\to\mathbb{R}$ is $\rho$-weakly convex iff 
$g+ \tfrac{\rho}{2}\|\cdot \|^2$ is convex. 
Hence, $\zeta_j$ with $z \mapsto \langle\mathbf{1},\psi_j(\mathbf z)\rangle$ is $1$-weakly convex because the $\psi_j$ are $1$-weakly convex.
Moreover, we have for any linear map $\mathbf{M}$ that
\begin{align}\label{eqLWeaklyConvProof}
&\zeta_j \circ \mathbf{M} +\tfrac{1}{2}\|\mathbf{M}\|^2\|\cdot\|^2 \notag\\
=&(\zeta_j+\tfrac{1}{2}\|\cdot \|^2) \circ \mathbf M +\tfrac{1}{2}(\|\mathbf{M}\|^2\|\cdot \|^2 - \|\mathbf M \cdot\|^2).
\end{align}
Due to $\|\mathbf M \mathbf x\|^2\le \|\mathbf{M}\|^2\|\mathbf x\|^2$, the second summand in \eqref{eqLWeaklyConvProof} is convex.
Thus, the composition $\zeta_j \circ \mathbf{M}$ is $\|\mathbf{M}\|^2$-weakly convex.
Now, we choose $\mathbf{M}=\mathbf W_j \mathrm R$ with $\|\mathrm R\|=1$ and $\|\mathbf W_j\|\le \| \mathbf W\|=1$.
Then each functional $\zeta_j(\mathbf W_j \mathrm R \cdot)$ is $1$-weakly convex. 
As non-negative average of $1$-weakly convex functions, $\mathcal{R}$ is $1$-weakly convex.
\end{proof}
Intuitively, weak convexity lets the learned regularizer be nonconvex, but only in a controlled way.
This offers a compromise between convex regularizers, which stabilize the optimization of \eqref{optimization} but are often too restrictive, and unconstrained learned priors, which are expressive but may render the optimization unstable. Proposition \ref{prop_weak_conv} makes this control precise.

If the $\psi_j$ are differentiable, we have that
\begin{equation}
    \nabla\mathcal{R}(\mathbf x)
=|\mathcal G|^{-1}\sum_{\mathrm R \in \mathcal G} \mathrm R^\top \mathbf W^\top \psi' (\mathbf W \mathrm R \mathbf x).
\end{equation}
For any $\mathbf x, \mathbf z \in \mathbb{R}^{2N}$, we then get
\begin{align}\label{eq:LipEst1}
    &\|\nabla\mathcal{R}(\mathbf x)-\nabla\mathcal{R}(\mathbf z)\|\notag\\ \leq&|\mathcal G|^{-1}\sum_{\mathrm R \in \mathcal{G}} \|\mathrm R^\top \| \|\mathbf W^\top \|
\| \psi'(\mathbf W \mathrm R \mathbf x)-\psi'(\mathbf W \mathrm R \mathbf z)\|.
\end{align}
Thus, if the $\psi_j'$ are Lipschitz continuous, we get from \eqref{eq:LipEst1} that the same holds for $\nabla R$, and by the same occasion also for the objective in \eqref{optimization}. Consequently, gradient-based optimization algorithms, such as Gradient Descent and its variants, converge to a critical point of \eqref{optimization}.

\subsection{Parametrization} \label{subsec:parametrization}
The architecture is illustrated in Figure~\ref{fig1}.
As trade-off between orientation diversity and computational efficiency, we adopt  
\begin{equation}
    \mathcal{G} = \{\mathrm{Id}, \mathrm{Rot}_{(X,\frac{\pi}{2})}, \mathrm{Rot}_{(Y,\frac{\pi}{2})}, \mathrm{Rot}_{(Z,\frac{\pi}{2})}\},
\end{equation}
i.e., the identity together with the three 90-degree rotations around the coordinate axes.
Further, we learn (i) the potentials $\{\psi_j\}_{j=1}^{J}$, 
and (ii) the convolution operators $\{\mathbf W_j\}_{j=1}^{J}$.
Below, we detail their parameterization.

\begin{figure}[htb]
\centerline{\includegraphics[width=\columnwidth]{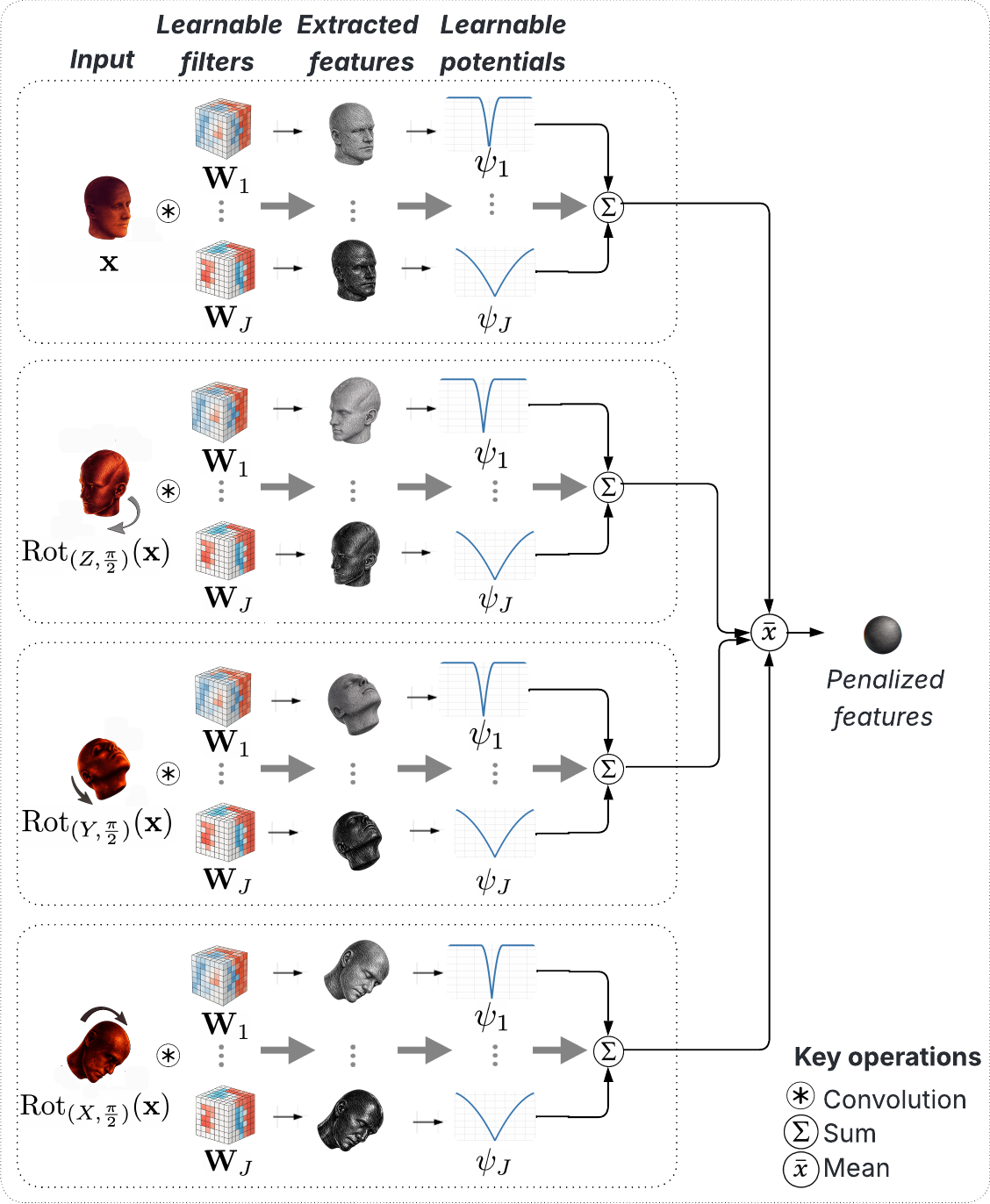}}
\caption{\, \capFigOne}\label{fig1}
\end{figure}

\paragraph{Potentials}
As proposed in \cite{LearnedRegularizers}, we take the Huber function with shape parameter $\beta>0$ given by
\begin{align}
\label{eq:huber}
\phi_{\beta}^{\text{Huber}}(t)
&=
\begin{cases}
\frac{\beta}{2} t^2, & |t|\le \beta^{-1},\\[2pt]
|t|-\frac{1}{2\beta}, & |t|> \beta^{-1},
\end{cases}
\end{align}
and define the 1-weakly convex potential 
\begin{equation}
\label{shared_potential}
    \phi_\beta = \phi_{\beta}^{\text{Huber}} - \phi_{1}^{\text{Huber}}.
\end{equation}
Following \cite{goujon2024learning}, we adapt $\phi_\beta$ per channel as
\begin{equation}
\label{potentials}
\psi_j = \frac{1}{\alpha_j^2} \phi_\beta(\alpha_j \cdot),
\qquad \alpha_j > 0,
\end{equation}
with learnable parameters $\{\beta\} \cup \{\alpha_j\}_{j=1}^J$.
Using a shared profile makes $\mathcal{R}$ more interpretable and easier to analyze.

\paragraph{Convolutions}
To enforce $\Vert \mathbf{W} \Vert_2 =1$, we proceed as in \cite{goujon2024learning} and set
\begin{equation}
\mathbf{W} = \frac{\mathbf{U}}{\|\mathbf{U}\|},
\end{equation}
where $\mathbf{U}$ denotes a 3D convolution operator.
The spectral norm $\|\mathbf{U}\|$ is estimated using the discrete Fourier transform.
Although this technically requires periodic boundary, we found it to be sufficiently accurate.

Inspired by the design in \cite{goujon2024learning}, the operator $\mathbf{U}$ is implemented as a cascade of three 3D convolutions with $3 \times 3 \times 3$ kernels, yielding an effective receptive field of $7 \times 7 \times 7$. Compared with learning a dense \(7\times7\times7\) convolution directly, the cascade reduces the number of parameters and computations while still allowing the learned filters to extract features over a sufficiently large neighborhood.
The number of output channels of the cascade are $4$, $8$, and $16$, respectively.
All convolutions are bias-free and use unit stride and grouping. 
Moreover, the first-layer kernels have zero mean. 
Since a zero-mean filter annihilates constant image components, the first layer responds only to local variations, mirroring classical sparsifying transforms such as the finite differences of total variation and wavelet high-pass filters, which by design respond to variation rather than global intensity.

\subsection{Training}\label{sec:Training}
\begin{algorithm}[htb]
\caption{nmAPG}\label{nmapg}
\begin{algorithmic}[1]
\Require Initialization $\mathbf x_0$, initial Lipschitz estimate $L_1 >0$, $\delta >0$, $\eta \in (0,1)$, $\rho \in (0,1)$, line search steps $K_L$, tolerance $\epsilon > 0$.
\State Set $\mathbf z_1=\mathbf x_1=\mathbf x_0$, $t_1=1$, $t_0=0$, $q_1=1$, $c_1=\mathcal{J}(\mathbf x_1)$
\While{$\|\mathbf x_k - \mathbf x_{k-1}\| / \|\mathbf x_{k-1}\| \geq \epsilon$ \textbf{or} $k=1$}
    \State $\bar{\mathbf x}_k \gets \mathbf x_k + \frac{t_{k-1}}{t_k}(\mathbf z_k - \mathbf x_k) + \frac{t_{k-1} - 1}{t_k}(\mathbf x_k - \mathbf x_{k-1})$
    \If{$k>1$}
        \State $L_k \gets \frac{\langle\nabla \mathcal{J}(\bar{\mathbf x}_k) - \nabla \mathcal{J}(\bar{\mathbf x}_{k-1}), \nabla \mathcal{J}(\bar{\mathbf x}_k) - \nabla \mathcal{J}(\bar{\mathbf x}_{k-1}) \rangle}{\langle\nabla \mathcal{J}(\bar{\mathbf x}_k) - \nabla \mathcal{J}(\bar{\mathbf x}_{k-1}), \bar{\mathbf x}_k - \bar{\mathbf x}_{k-1} \rangle}$
    \EndIf
    \State $c_k' \gets \max\{\mathcal{J}(\bar{\mathbf x}_k),c_k\}$
    \For{$l = 1, \ldots, K_L$}
        \State $\mathbf z_{k+1} \gets \bar{\mathbf x}_k - \frac{1}{L_k} \nabla \mathcal{J}(\bar{\mathbf x}_k)$
        \If{$\mathcal{J}(\mathbf z_{k+1}) \leq c_k'  - \delta \|\mathbf z_{k+1} - \bar{\mathbf x}_k\|^2$}
            \State \textbf{break}
        \EndIf
        \State $L_k \gets \frac{L_k}{\rho}$
    \EndFor
    \If{$\mathcal{J}(\mathbf z_{k+1}) \leq c_k  - \delta \|\mathbf z_{k+1} - \bar{\mathbf x}_k\|^2$}
        \State $\mathbf x_{k+1} \gets \mathbf z_{k+1}$
    \Else
        \State $L_k \gets \frac{\langle\nabla \mathcal{J}(\mathbf x_k) - \nabla \mathcal{J}(\bar{\mathbf x}_{k-1}), \nabla \mathcal{J}(\mathbf x_k) - \nabla \mathcal{J}(\bar{\mathbf x}_{k-1}) \rangle}{\langle\nabla \mathcal{J}(\mathbf x_k) - \nabla \mathcal{J}(\bar{\mathbf x}_{k-1}), \mathbf x_k - \bar{\mathbf x}_{k-1} \rangle}$
        \For{$l = 1, \ldots, K_L$}
            \State $\mathbf v_{k+1} \gets \mathbf x_k - \frac{1}{L_k} \nabla \mathcal{J}(\mathbf x_k)$
            \If{$\mathcal{J}(\mathbf v_{k+1}) \leq c_k  - \delta \|\mathbf v_{k+1} - \mathbf x_k\|^2$}
                \State \textbf{break}
            \EndIf
            \State $L_k \gets L_k / \rho$
        \EndFor
        \If{$\mathcal{J}(\mathbf z_{k+1}) \le \mathcal{J}(\mathbf v_{k+1})$}
            \State $\mathbf x_{k+1} \gets \mathbf z_{k+1}$
        \Else
            \State $\mathbf x_{k+1} \gets \mathbf v_{k+1}$
        \EndIf
    \EndIf
    \State $t_{k+1} \gets \frac{\sqrt{4t_k^2+1}+1}{2}$
    \State $q_{k+1} \gets \eta q_k +1$
    \State $c_{k+1} \gets \frac{\eta q_k c_k + \mathcal{J}(\mathbf x_{k+1})}{q_{k+1}}$
    \State $k \gets k + 1$
\EndWhile
\State \Return minimizer $\mathbf x_k$
\end{algorithmic}
\end{algorithm}

\begin{figure*}[t]
\centering
    \includegraphics[width=\linewidth]{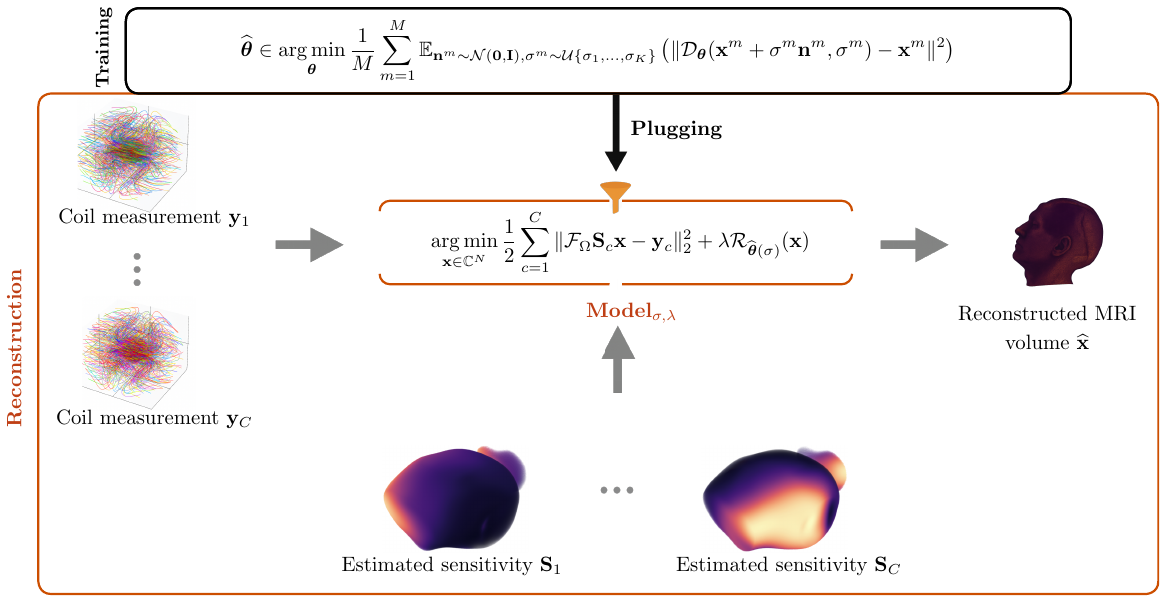}
    
\caption{\, \capFigTwo}
\label{train_recon_pipeline}
\end{figure*}
As training task for the regularizer $\mathcal{R}_\theta$ in \eqref{reg} with aggregated parameters 
$\theta$, we consider denoising.
To this end, let $\{\mathbf{x}^m\}_{m=1}^{M}$ denote a collection of clean training volumes.
Each volume $\mathbf{x}^m$
is corrupted as $\mathbf{y}^m = \mathbf{x}^m + \sigma^m \mathbf{n}^m$ with Gaussian noise $\mathbf{n}^m \sim \mathcal{N}(\mathbf{0},\mathbf{I})$ and noise level $\sigma^m \in [\sigma_{\min},\sigma_{\max}]$.
Inserting the parametric regularizer $\mathcal{R}_\theta$ from \eqref{reg} into \eqref{optimization} (without the operator $\mathbf A$) induces the denoiser
\begin{equation}
\mathcal D_{\theta}(\mathbf y)
=\argmin_{\mathbf x\in\mathbb{R}^{2N}}
\frac{1}{2}\|\mathbf x- \mathbf y\|_2^2 + \mathcal{R}_\theta(\mathbf x).
\label{prox_denoiser}
\end{equation}
Now, to deal with all $\sigma \in [\sigma_{\min}, \sigma_{\max}]$ simultaneously, we follow \cite{goujon2024learning} and condition the $\alpha_j$ in \eqref{potentials} on $\sigma$ via the modified parametrization
\begin{equation}
\alpha_j(\sigma) 
= \frac{\exp(s_{c_j}(\sigma))}{\sigma + 10^{-5}},
\qquad \sigma \in [\sigma_{\min}, \sigma_{\max}],
\label{alpha_cond}
\end{equation}
where $s_{c_j}$ is a learnable linear spline with $K$ equidistant knots $\sigma_1=\sigma_{\min}, \ldots, \sigma_K=\sigma_{\max}$ and associated (learnable) values $c_j$.
This leads to the conditioned regularizer $\mathcal{R}_{\theta(\sigma)}$ with associated denoiser
\begin{equation}
\label{noise_cond_prox}
\mathcal{D}_{\theta}(\mathbf y, \sigma) =
\argmin_{\mathbf x\in\mathbb{R}^{2N}}
\frac{1}{2}\|\mathbf x - \mathbf y\|_2^2 
+ \mathcal{R}_{\theta(\sigma)}(\mathbf x).
\end{equation}
To minimize the objective in \eqref{noise_cond_prox}, we choose the \emph{non-monotone accelerated proximal gradient} (nmAPG) algorithm, which ensures convergence to a critical point \cite[Supp.\ Thm.\ 4]{li2015accelerated}.
The scheme is summarized for a generic objective $\mathcal J$ in Algorithm~\ref{nmapg}.
At each step $k$, a candidate $\mathbf z_{k+1}$ for the next iterate $\mathbf x_{k+1}$ is obtained through a gradient step taken from the extrapolation $\bar{\mathbf x}_k$.
If $\mathbf z_{k+1}$ does not satisfy the non-monotone acceptance criterion, we compute a fallback candidate $\mathbf v_{k+1}$ without the extrapolation and choose $\mathbf x_{k+1}$ as the one with the lower energy $\mathcal J$.
In both cases, the step size is initialized using a Barzilai–Borwein rule \cite{barzilai1988two} and subsequently refined by a backtracking line search.
In our experiments, we set $\delta=0.1$, $\eta=0.8$ and $\rho=0.9$, and terminate the iterations when the relative change between consecutive iterates falls below 
$\epsilon = \num{1e-4}$, unless otherwise stated.

For the multi-noise denoiser $\mathcal{D}_{\theta}$, we seek the parameters $\theta$ that minimize the average reconstruction error over the dataset and all noise levels, namely
\begin{equation}
\label{eq:training}
\hat{\theta} \in 
\argmin_{\theta}
\frac{1}{M} \sum_{m=1}^{M}
\mathbb{E}_{(\mathbf{n}^m,\sigma^m)}
\|\mathcal D_{\theta}(\mathbf{y}^m, \sigma^m)
- \mathbf{x}^m\|_2^2.
\end{equation}
We minimize the training loss \eqref{eq:training} using the stochastic AdaBelief optimizer \cite{zhuang2020adabelief}.
This involves the Jacobian $\mathbf J_{\theta} \mathcal D_{\theta}(\mathbf{y}, \sigma)$.
To compute it, we adopt the implicit differentiation approach popularized for deep equilibrium models \cite{bai2019deep}.
Specifically, we have that $\hat{\mathbf x}( \theta) = \mathcal D_{\theta}(\mathbf{y}, \sigma)$ satisfies
the fixed-point condition
\begin{equation}
\label{eq:fixed_point}
\hat{\mathbf{x}}(\theta) - \mathbf{y}
+ \nabla_{\mathbf{x}} \mathcal{R}_{\theta}(\hat{\mathbf{x}}(\theta)) = 0.
\end{equation}
Applying the implicit-function theorem  to~\eqref{eq:fixed_point} yields
\begin{equation}
\label{eq:implicit_jacobian}
(\mathbf{I} + \mathbf{H}_{\mathcal{R}}(\theta,\hat{\mathbf{x}}(\theta)))
 \mathbf J_{\theta}\hat{\mathbf{x}}(\theta)
= - \mathbf J_{\theta}(\nabla_{\mathbf{x}}\mathcal{R})(\theta,\hat{\mathbf{x}}(\theta)),
\end{equation}
where $\mathbf{H}_{\mathcal{R}}$ denotes the Hessian of $\mathcal R_{\theta}$ with respect to $\mathbf{x}$.
Matrix–vector products involving $\mathbf J_{\theta}\mathcal D_{\theta}$ are computed by solving the linear system \eqref{eq:implicit_jacobian} using the \emph{minimum residual method}.

Training is performed with a batch size of $12$ for $500$ epochs using an initial learning rate of $\num{1e-2}$ and an exponential learning-rate schedule with a decay factor of $0.05^{1/500}$ per epoch.
We use $M=47$ training volumes (see Section~\ref{sec:Pipeline}) and randomly extract $64 \times 64 \times 64$ patches for each batch to improve computational efficiency.
Regarding the noise range, we set $\sigma_{\min} = 0.001$, $\sigma_{\max} = 0.1$ and  $K=12$.
Within each batch, every element is corrupted with a different noise level corresponding to one of the 12 spline knots.
This maximizes the coverage of noise information, which we found to accelerate the training.
This strategy can be extended to subsampling without replacement if the spline contains more knots than batch elements.
Moreover, we regularize the training loss in \eqref{eq:training} with \begin{equation}
\label{hessian_reg}
\mu \left \Vert \mathbf{H}_{\mathcal{R}}(\theta,\hat{\mathbf{x}}(\theta)) \right \Vert_2,
\end{equation}
namely the spectral norm of the Hessian  of $\mathcal{R}$, with regularization strength $\mu = \num{1e-6}$ every five steps.
The spectral norm is approximated using up to 50 power iterations.
Empirically, this smoothness-promoting regularization resulted in parameter configurations for which nmAPG converged faster during evaluation.

\subsection{MRI reconstruction}\label{sec:Pipeline}

Given the regularizer \eqref{reg}, which is trained for denoising as described in Section~\ref{sec:Training}, we now want to perform MRI reconstruction using variational reconstruction.
As before, we minimize the reconstruction objective \eqref{optimization} with Algorithm~\ref{nmapg}.
Since the model was trained for denoising, we have to tune both the regularization parameter $\lambda$ and the noise level $\sigma$ of $\mathcal R$.
A summary of the complete training and reconstruction pipeline is given in Figure~\ref{train_recon_pipeline}.
Details on the data, the deployed sampling trajectory, our benchmark setup, and the hyperparameter tuning are provided below.

\subsubsection{Data}\label{sec:data}

This study was conducted retrospectively using human subject data made available in open access \cite{beauferris2022multi}.
Ethical approval was not required, as confirmed by the license attached with the data.
All the prospective acquisitions were performed on a healthy volunteer with approvals from local and national ethical committees for the protocol, and after that, a written consent was obtained.
The respective data on which results are presented (whether retrospectively or prospectively acquired) are explicitly stated in the results section.
In the following, we describe the training and validation data.

We use the Cartesian raw k-space data from the \textit{Calgary-Campinas} dataset~\cite{beauferris2022multi}, 
which provides 167 volumetric 3D T$_1$-weighted gradient-echo brain scans acquired from healthy volunteers on a 3T scanner.  
The dataset includes 117 volumes collected with a 12-channel coil and 50 volumes acquired using a 32-channel coil.
The train/validation/test split of the original dataset was 47/20/50 for 12-coil measurements and all the 50 32-coil measurements were used for test.
All volumes admit partial Fourier sampling (up to 85\%) along the slice-encoding direction ($k_z$).
For numerical stability, we divide the measurements by $10^6$.

From these, we generated complex reference volumes by zero-filling the non-sampled regions, followed by an inverse fast Fourier transform.
For the resulting per-coil volumes, we applied virtual coil combination to obtain reference volumes.
To ensure consistent dimensions, these are center-cropped to size $N_x \times N_y \times N_z = 256 \times 218 \times 170$.
In the absence of a publicly available, high-quality 3D MRI database, we consider these volumes as a reasonable collection for training and benchmarking purposes without further post-processing.

\begin{figure}[ht]
    \centering
    
    \includegraphics[width=0.99\linewidth]{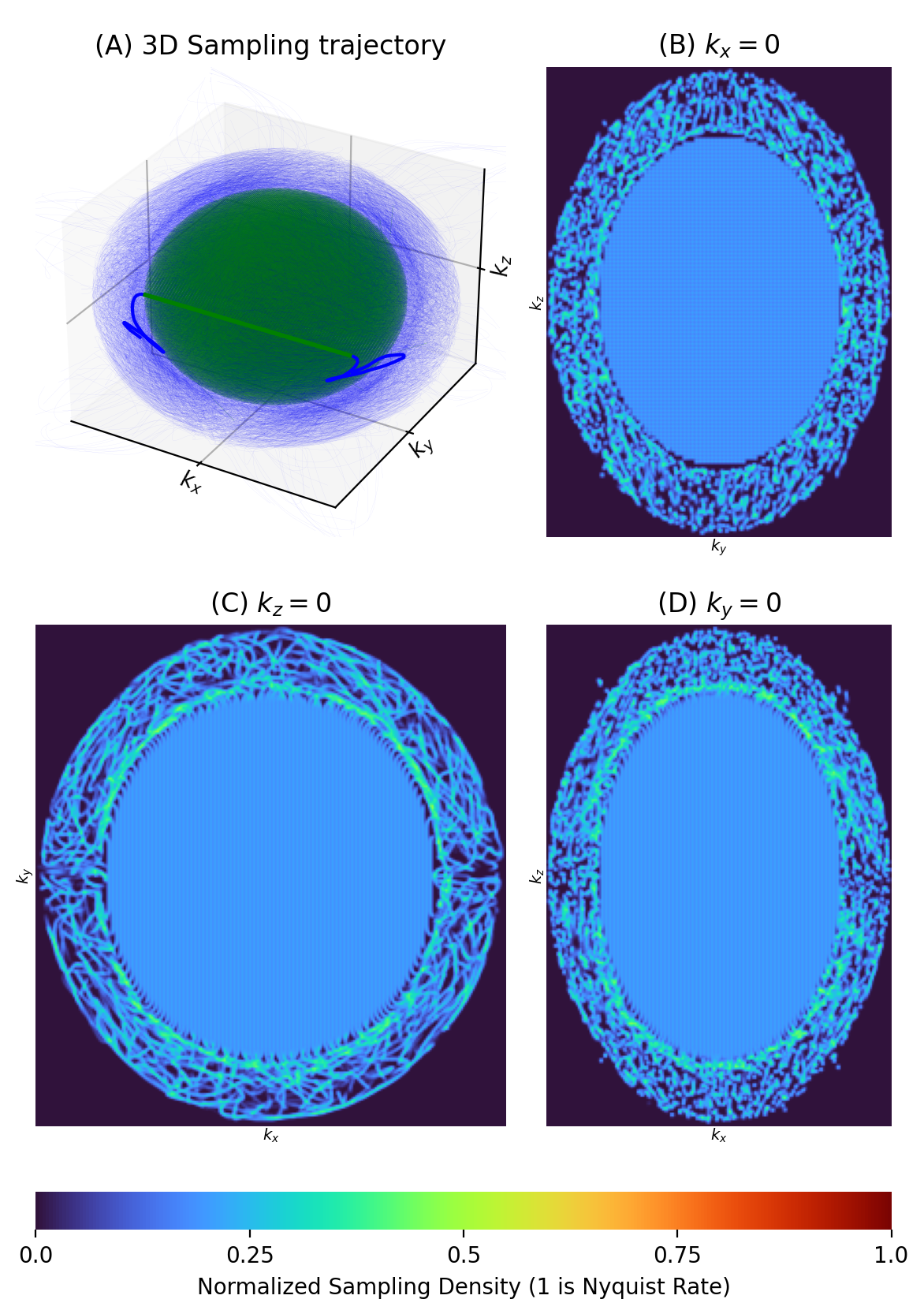}
    \caption{\, \capFigThree}
    \label{trajectory}
\end{figure}

\subsubsection{Non-cartesian sampling trajectory}
In our experiments, we adopt the 3D GoLF-SPARKLING trajectory with GRAPPA acceleration \cite{radhakrishna2025bringing}, which relies on two complementary strategies: Cartesian undersampling of the k-space center and non-uniform sampling at higher frequencies.
To this end, the trajectory readouts are designed to intersect the k-space center along straight Cartesian lines. 
Then, by modifying trajectory-specific affine constraints \cite{3dsparkling,giliyar2023improving}, we ensure a structured, GRAPPA-like undersampling pattern for the center.
In practice, we maintain a 2x2 pattern, concentrating approximately 60\% of the measurements near the origin.
In contrast, the remaining portion of the k-space follows a non-Cartesian sampling distribution to exploit the benefits of compressed sensing.

The trajectory consists of $4489$ shots (readouts) with $416$ samples each.
A visualization is provided in Figure~\ref{trajectory}.
The chosen design ensures robust low-frequency coverage and data consistency while preserving the efficiency of non-Cartesian SPARKLING sampling at high frequencies.
It achieves an acceleration factor of $\rm{AF} \approx 8.3$, resulting in a scan time of nearly 1 minute for a 1 mm isotropic whole brain MRI scan.

\subsubsection{Benchmark setup}\label{sec:Experiments}

We compare WCRR to several representative MRI reconstruction methods: parallel imaging (GRAPPA / SENSE), compressed sensing ($\ell_1$-wavelet and TV), image domain postprocessing (DRUNet3D), unrolled networks (NC-PDNet), plug-and-play (DPIR) and diffusion models (Score-DiffMBIR).
Below are several considerations/observations that apply to all methods, including WCRR:

\begin{enumerate}[label=\roman*.]
    \item For quantitative evaluation, following the fastMRI protocol \cite{Muckley2021}, PSNR and the structural similarity index (SSIM)~\cite{wang2004image} were computed only within a foreground mask, defined as the voxels whose magnitude exceeds $5\%$ of the maximum ground-truth magnitude.
    This restricts evaluation to anatomical regions and prevents bias from the large background area. 
    The resulting metrics are referred to as "masked PSNR" and "masked SSIM".

    \item For each method with tunable hyperparameters, the optimal values were selected by a \textbf{grid search} for the best masked PSNR on our validation set consisting of 5 12-coil volumes from the validation split.
    For WCRR, this yields $\sigma = 0.03$ and $\lambda = 0.01$; the sensitivity of these two parameters is investigated in the supplementary material.

    \item Whenever needed, the sensitivity maps $\mathbf S_c$ are estimated using ESPIRiT \cite{uecker2014espirit} as implemented in \emph{MRI-NUFFT} with calibration width of \(24\), kernel width \(6\), singular-value threshold \(0.02\), eigenvalue crop threshold \(0.95\), and decimation factor \(4\) for memory-efficient estimation.

    \item For all iterative methods, the initialization $\mathbf x_0$ was taken as GRAPPA + DCp. We terminated all energy-based methods when the relative change between consecutive iterates drops below $\num{5e-3}$, except for TV which required a smaller tolerance of $\num{5e-4}$.
    Empirically, after these thresholds, the reconstruction metrics remained constant.

    \item For their implementation, we rely on the \emph{DeepInverse} library \cite{tachella2025deepinverse}. Moreover, we use the 3D non-uniform fast Fourier transform (NUFFT) implementation from \emph{MRI-NUFFT} \cite{comby2024mri}. Our code is available on Github\footnote{\url{https://github.com/Shamachrist7/wcrr-noncartesian-3d-mri}}.
\end{enumerate}

\paragraph{GRAPPA + DCp / SENSE}

GRAPPA performs a kernel-based interpolation to fill the missing values in the gridded region of the k-space using neighboring values.
We deploy a kernel with size $N_{\mathbf{k}_x} \times N_{\mathbf{k}_y} \times N_{\mathbf{k}_z} = 5 \times 4 \times 4$.
The GRAPPA kernel weights were fitted based on the central $24\times24$ k-space lines, which correspond to the autocalibration signal (ACS) along with Tikhonov regularization with regularization parameter $10^{-4}$.
We then apply the density compensation (DCp) adjoint as approximate inverse.
The DCp weights are estimated iteratively as described in \cite{pipe1999sampling} with potential generalizations discussed in \cite{kircheis2023fast}.

The parallel imaging baselines SENSE and GRAPPA perform similarly. Thus, we only present the results of one of the two to avoid redundancy. Here, SENSE solves the optimization problem in \eqref{optimization} without regularization term and using $10$ iterations of the conjugate gradient algorithm.

\paragraph{$\ell_1$-wavelet}
This variational method is based on the regularizer $\mathcal{R}(\mathbf x) = \| [\mathbf{\Psi} \mathbf x]_{\text{details}} \|_1$, where $\mathbf{\Psi}$ denotes the orthonormal 3D Daubechies-4 wavelet transform \cite{daubechies1988orthonormal} with four decomposition levels (ignoring the  approximation coefficients).
The associated reconstruction problem \eqref{optimization} is solved using the \emph{Fast Iterative Shrinkage-Thresholding Algorithm} (FISTA) \cite{Beck2009FISTA}.
The required proximal operator of $\mathcal R$ has a closed-form (wavelet thresholding).
Here, the optimal regularization strength is $\lambda = \num{3e-3}$.

\paragraph{TV} For reconstruction with isotropic TV regularization, we deploy a primal-dual scheme \cite{Condat2013} with updates
\begin{align}
\mathbf p_{k+1} &= \operatorname{proj}_{\lambda \mathcal B_{2,\infty}}\bigl(\mathbf p_k + \eta \nabla \mathbf x_k \bigr)\\
\mathbf x_{k+1} &= \mathbf x_k - \tau\bigl(\mathbf A^T(\mathbf A \mathbf x_k - \mathbf y) + \nabla^{\top} (2 \mathbf p_{k+1} - \mathbf p_k)\bigr),\notag
\end{align}
where $\mathcal B_{2,\infty}$ denotes the unit ball for the grouped norm $\Vert \cdot \Vert_{2,\infty}$.
To ensure convergence, we set $\tau = 1/\|A\|^2$, $\eta = 1/(24\tau )$ and $\mathbf p_0 = \mathbf 0 \in \mathbb{R}^{2N \times 3}$.
Here, the optimal regularization strength is $\lambda = \num{1e-4}$.

\paragraph{Post-processing: DRUNet3D}
We train a supervised image-domain post-processing baseline based on a 3D extension of DRUNet \cite{ZhaLiZuo2022}.
A version trained on our data for denoising across a continuum of noise levels is publicly available on Hugging Face\footnote{\url{https://huggingface.co/deepinv/drunet_3d_denoise_complex/tree/main} \label{hugging_face_drunet}}.
We initialize the network with these weights, fix the conditional noise-level input to zero, and train it end-to-end to map the adjoint NUFFT reconstruction onto the corresponding reference image, using an MSE loss on the magnitudes of prediction and target.
Further training details are given in the supplementary material.

\paragraph{NC-PDNet}
The NC-PDNet \cite{ramzi2022nc} is a 3D non-Cartesian extension of XPDNet, which ranked second in the fastMRI challenge \cite{Muckley2021}.
It consists of 6 unrolled iterations of the Chambolle-Pock algorithm \cite{chambolle2011first} with uncoupled parameters.
As refinement backbone, we adopt a residual 3D U-Net \cite{Ronneberger2015}, with three scales, SiLU activations, and 16–32–64 channels per scale, respectively.
Each LPD step takes the previous two updates as input (extrapolation), and the training loss consists of a combined Mean Absolute Error (MAE) and Structural Similarity Index (SSIM) on magnitude differences, defined as $\mathcal{L} = 5 \cdot \text{MAE} + \frac{1 - \text{SSIM}}{2}$. For complementary training details, see the supplementary material.

\paragraph{Plug-and-play: DPIR}
Here, we consider again the above-mentioned DRUNet denoiser.
Following \cite{ZhaLiZuo2022}, we unroll the \emph{Half Quadratic Splitting} algorithm for $K=8$ iterations, leading to

\begin{equation}
\begin{array}{l}
\mathbf u_{k} = \operatorname{prox}_{\gamma_k \Vert \mathbf A \cdot -\mathbf y \Vert^2 /2} (\mathbf x_k)\\
\mathbf x_{k+1} = \mathcal{D}(\mathbf u_k, \sigma_k),
\end{array}
\end{equation}
with noise levels $\sigma_k 
= \sigma_{\max}(\sigma_{\min}/ \sigma_{\max})^{\frac{k}{K-1}}$ and stepsizes
$\gamma_k 
= \lambda (\sigma_k / \sigma_{\min})^2$, where the optimal values are $\sigma_{\max} = 0.015$, $\sigma_{\min} = \num{2e-3}$ and $\lambda = 6.0.$

\paragraph{Score-DiffMBIR}
As diffusion baseline, we adopt the 2D score network architecture \emph{NCSN++} of \cite{chung2022score}, adapted to our complex data:
two (real-imaginary) input and output channels, 96-192-192-384 output channels across the downsampling stages, and self-attention at resolution 32.
The network is trained on 2D slices along the $z$ (last) dimension of our data with a denoising score matching loss; see the supplementary material for further training details.
Then, we deploy the DiffMBIR strategy~\cite{chung2023solving} to perform reconstruction with this 2D model by alternating $K$ times between the diffusion-prior update and three steps of the numerical solver applied to ${\Vert \mathbf A \cdot -\mathbf y \Vert^2 /2} + \lambda \Vert \nabla_z \mathbf \cdot \Vert_1$. The corresponding diffusion-prior noise levels are equispaced across the number of iterations $K$ between $\sigma_K = \sigma_{\max}$ and $\sigma_1 = \sigma_{\min}$. The optimal values are $\sigma_{\max} = 0.01$, $\sigma_{\min} = \num{2e-3}$, $\lambda = \num{1e-2}$ and $K = 6$.

\section{Results}

All experiments were conducted on an \emph{NVIDIA RTX PRO 6000 Blackwell} with 96 GB memory.

\subsection{Retrospective simulation results}\label{sec:SimulatedExp}
\begin{table*}[ht]
\centering
\resizebox{\textwidth}{!}{
\begin{tabular}{lcccccccc}
\toprule
\multirow{3}{*}{\textbf{Method}}
& \multicolumn{3}{c}{\textbf{12-coil}} & \multicolumn{3}{c}{\textbf{32-coil}} & \multirow{3}{*}{\textbf{\# Param.}} & \multirow{3}{*}{\textbf{Train. Time}} \\
\cmidrule(lr){2-4} \cmidrule(lr){5-7}
& Masked & Masked & \multirow{2}{*}{Recon. Time} 
& Masked & Masked & \multirow{2}{*}{Recon. Time} 
&  \\
& PSNR & SSIM &  
& PSNR & SSIM &  
&  \\
\midrule

GRAPPA + DCp
& $28.72 \pm 1.91$ & $0.8901 \pm 0.00194$ & \underline{5\,s} 
& $35.47 \pm 0.91$ & $0.9601 \pm 0.0077$ & \underline{11\,s} 
& -- & --\\
\midrule

$\ell_1$-wavelet 
& $30.30 \pm 1.54$ & $0.9233 \pm 0.0123$ & 37\,s 
& $36.21 \pm 0.82$ & $0.9732 \pm 0.0050$ & 1\,min\,21\,s 
& -- & --\\

TV 
& $31.56 \pm 1.59$ & $0.9402 \pm 0.0093$ & 1\,min\,10\,s
& $37.14 \pm 0.82$ & $0.9786 \pm 0.0039$ & 2\,min\,23\,s 
& -- & --\\
\midrule

DRUNet3D 
& $31.02 \pm 1.30$ & $0.9394 \pm 0.0059$ & \textbf{2\,s} 
& $34.61 \pm 0.72$ & $0.9699 \pm 0.0051$ & \textbf{4\,s}
& 96{,}543{,}168 & 41\,h\\
\midrule

NC-PDNet 
& $31.40 \pm 1.41$ & $0.9438 \pm 0.0071$ & 7\,s 
& $35.29 \pm 0.64$ & $0.9777 \pm 0.0034$  & 19\,s 
& \underline{6{,}670{,}584} & 72\,h\\
\midrule

DPIR 
& $31.67 \pm 1.36$ & $\underline{0.9519} \pm 0.0080$ & 4\,min\,49\,s 
& $35.23 \pm 1.59$ & $0.9785 \pm 0.0041$ & 8\,min\,43\,s 
& 96{,}543{,}168 & \textbf{11\,h}\\
\midrule

Score-DiffMBIR 
& $\mathbf{32.54} \pm 1.57$ & $\mathbf{0.9520} \pm 0.0073$ & 36\,s 
& $\mathbf{37.54} \pm 0.80$ & $\mathbf{0.9820} \pm 0.0033$ & 1\,min 4\,s 
& 69{,}660{,}584 & 54\,h\\
\midrule

WCRR \textbf{(ours)} 
& $\underline{32.04} \pm 1.54$ & $0.9469 \pm 0.0077$ & 34\,s 
& $\underline{37.36} \pm 0.80$ & $\underline{0.9810} \pm 0.0035$ & 1\,min\,2\,s
& \textbf{4{,}729} & \underline{13\,h}\\

\bottomrule
\end{tabular}
}
\caption{\capTabTwo}
\label{tab:quantitative_results}
\end{table*}

\begin{figure*}[ht]
    \centering
    \includegraphics[width=\linewidth]{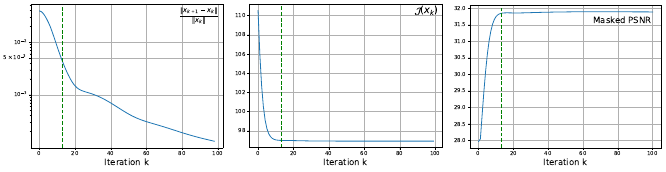}
    
    \caption{\, \capFigFour}
    \label{fig:convergence_curves}
\end{figure*}

We generate non-Cartesian undersampled k-space data directly from our preprocessed per-coil MR images, see Section \ref{sec:data}, and add white Gaussian noise with a standard deviation of $\sigma = \num{2e-3}$. For computational efficiency, we restrict ourselves to the first 15 measurements in both the 12- and 32-coil test datasets (a total of 30), selected in alphabetical order by name.
We recall that the models are only fitted for the 12-coil setup and the 32-coil setup remains as a generalization task.

\paragraph{Quantitative results}

Table~\ref{tab:quantitative_results} summarizes the retrospective results for the 12-coil and 32-coil test sets. Among the non-learning baselines, GRAPPA~+~DCp is computationally efficient but gives the lowest image quality, while \(\ell_1\)-wavelet and TV improve the reconstruction metrics at the cost of longer reconstruction times. Among the learned baselines, DRUNet3D is the fastest method in the table, but its purely image-domain post-processing does not reach the accuracy of the iterative learned-prior methods. NC-PDNet improves over DRUNet3D but remains below the best-performing methods in this limited-data 3D setting. DPIR gives strong SSIM values, but its repeated denoiser calls make it the slowest reconstruction method.

Score-DiffMBIR obtains the best masked PSNR and SSIM in both coil settings but requires being trained four times longer than WCRR. WCRR is consistently second-best in masked PSNR and, for the 32-coil setting, also second-best in masked SSIM. The advantage of Score-DiffMBIR over WCRR is 0.49 dB in PSNR and 0.0051 in SSIM for the 12-coil data, and decreases to 0.18 dB and 0.0010, respectively, for the 32-coil data. Thus, although Score-DiffMBIR achieves better scores, WCRR remains competitive and particularly close in the unseen 32-coil setting, while using roughly three orders of magnitude fewer trainable parameters than NC-PDNet and about four orders of magnitude fewer than DRUNet3D, DPIR, and Score-DiffMBIR. Additionally, its reconstruction time is comparable to Score-DiffMBIR and substantially shorter than DPIR, although not as short as the direct learned mappings DRUNet3D and NC-PDNet.

\begin{figure*}[ht]
\centering
    \includegraphics[width=\linewidth]{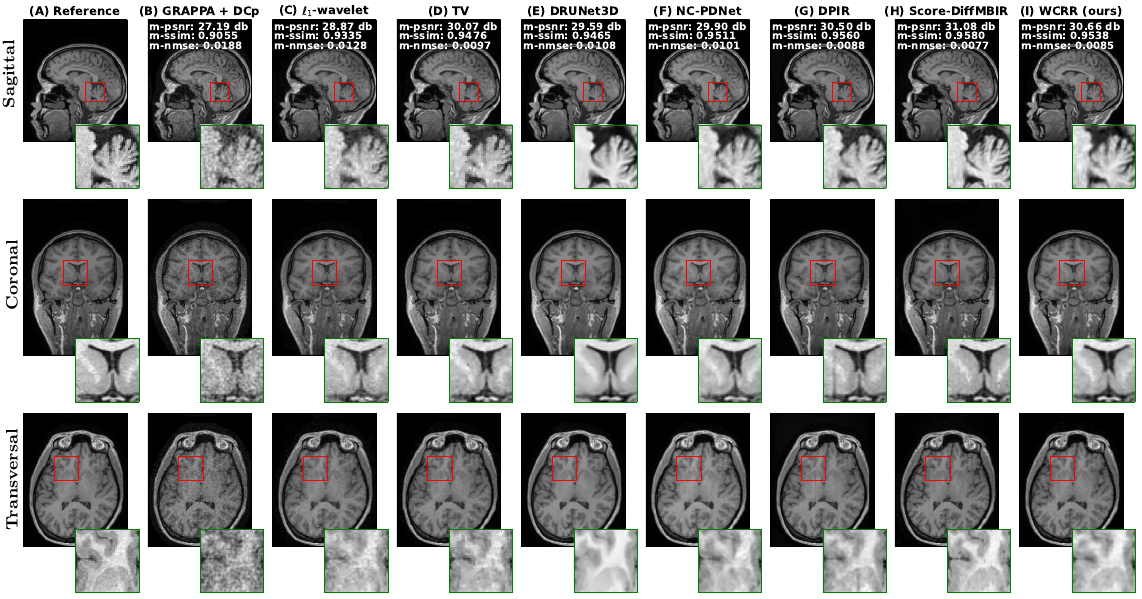}

\caption{\, \capFigFive}
\label{fig:example_reconstuctions}
\end{figure*}

\paragraph{Qualitative results}

Figure~\ref{fig:example_reconstuctions} shows a representative 12-coil reconstruction in the same method order as Table~\ref{tab:quantitative_results} through sagittal, coronal, and transversal slices. GRAPPA~+~DCp leaves visible residual artifacts, while \(\ell_1\)-wavelet and TV reduce aliasing but appear with their usual artifacts (block discontinuities and staircasing, respectively). DRUNet3D and NC-PDNet improve the appearance of several structures, but residual errors remain relatively high as confirmed by the additionally  reported masked normalized mean squared error (m-nmse). DPIR, Score-DiffMBIR, and WCRR give the best visual reconstructions, which are more faithful to the ground truth reference in this example, although DPIR exhibits a slight structural artefact in the coronal and tranversal views. Score-DiffMBIR has the highest numerical score, while WCRR gives a visually comparable result with a much smaller learned model and the structure of a convergent variational reconstruction method, as illustrated by the convergence curves in Figure~\ref{fig:convergence_curves}.

\paragraph{Ablation}
\label{p:ablation}

To isolate the contribution of the rotations within WCRR, we retrain it under the same setting but with the rotation set $\mathcal{G}$ narrowed to $\{\mathrm{Id}\}$, and compare this non-rotation-invariant version to the rotation-invariant one across the 12-coil testing set.
Table \ref{tab:wcrr_ablation} shows that the rotation-invariant version clearly outperforms it.
As expected, matching its performance without rotations requires four times as many filters ($4 \times 16 = 64$). Additionally, incorporating the rotations on top of this many filters ($64$) does not contribute to significant improvement anymore, as the model performance starts saturating from a certain number of filters onwards. In fact, sufficiently many learnable filters would learn the different orientations that we covered explicitely.
The rotation set thus reduces model complexity by a factor of $4$ in our case while also strengthening our sparsity prior interpretation.

\begin{table*}[!ht]
\centering
\begin{tabular}{lcccc}
\hline
 \multirow{2}{*}{Version[\# \emph{learnable filters}]} & Masked & Masked & \multirow{2}{*}{Recon. Time} & \multirow{2}{*}{\# Param.} \\
 & PSNR & SSIM & \\
\hline
non-rotation-invariant[\emph{$16$ filters}]  & $31.82 \pm 1.56$ & $0.9446 \pm 0.0080$ & \textbf{28 \, s} & \textbf{4{,}729} \\
rotation-invariant[\emph{$16$ filters}] \textbf{(ours)} & $32.04 \pm 1.54$ & $\underline{0.9469} \pm 0.0077$ & 34 \, s & \textbf{4{,}729} \\
\hline
non-rotation-invariant[\emph{$64$ filters}] & $\underline{32.06} \pm 1.54$ & $\underline{0.9469} \pm 0.0077$ & \underline{32 \, s} &  \underline{15{,}673}\\
rotation-invariant[\emph{$64$ filters}] & $\textbf{32.11} \pm 1.53$ & $\textbf{0.9475} \pm 0.0075$ & 44 \, s &  \underline{15{,}673}\\
\hline
\end{tabular}
\caption{\, \capTabOne}
\label{tab:wcrr_ablation}
\end{table*}

\subsection{Out of distribution tests} \label{sec:OutofDist}
We now showcase out-of-distribution results. 
For these, acquisitions were performed during the SENIOR cohort \cite{Haeger2020} scans at Neurospin with a 3T Siemens Healthineers Prisma$^{fit}$ scanner and a 20-channel head coil.
We utilized a clinically standard MPRAGE sequence with 1 mm isotropic resolution and whole-brain coverage.
The parameters included TE/TR/TI of 3 ms/7.6 ms/800 ms, an MPRAGE-TR of 2.3 s, turbofactor of 176 and a flip angle of $9^\circ$.
For ground truth reference, a fully sampled k-space acquisition was performed with a total scan time of nearly 9 minutes.
For all the accelerated scans, external ACS acquisitions were obtained to estimate the sensitivity maps $\mathbf S_c$ through ESPIRiT \cite{uecker2014espirit}.
All comparisons are purely qualitative due to potential inter-scan motion and slight contrast changes due to custom sequence for our implementation, which skew quantitative metrics. We therefore use these prospective experiments only as qualitative illustrations of reconstruction behavior under acquisition changes. They are not intended to establish clinical superiority among the methods.
Additionally, the results of the postprocessing baseline DRUNet3D were omitted here, as the method was not performing well under the distribution shift. This was to be expected as postprocessing methods adapt to the exact distribution of the data on which they have been trained.

\paragraph{GS-SPARKLING with 2x2 GRAPPA}

We first tested all methods for acquisition based on the trajectory described in Section~\ref{sec:data} with a total scan time of 1 minute.
The results are given in Figure~\ref{fig:prospective_sparks} along with zoomed in panes in sagittal, coronal and transversal views for easier qualitative comparison. 

The learned reconstruction methods NC-PDNet, DPIR,
Score-DiffMBIR, and WCRR show fewer visible
artifacts than the conventional $\ell_1$-wavelet and TV reconstructions in this example, and better preserve the anatomical structures. There is, however, no clear qualitative winner among these learned methods. We reiterate that these observations are qualitative and should not be interpreted as evidence of clinical superiority.
\begin{figure*}[!t]
\centering
    \includegraphics[width=\linewidth]{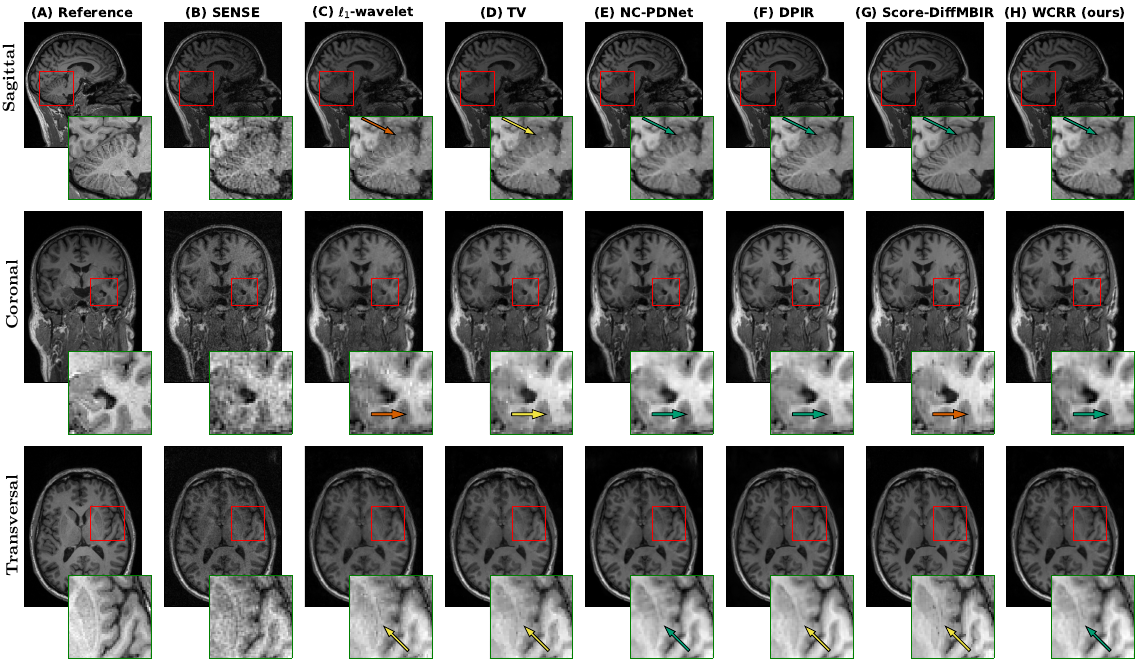}

\caption{\, \capFigSeven}
\label{fig:prospective_sparks}
\end{figure*}

\paragraph{CAIPIRINHA 3x2}
We also evaluated the models in a highly accelerated Cartesian setting using a 1-minute accelerated acquisition with $3 \times 2$ CAIPIRINHA \cite{caipi} (with $\Delta = 1$) and elliptical sampling.
The results are given in Figure \ref{fig:prospective_caipi}.

As the acquisition operator changed relative to the training setup, NC-PDNet shows a significant drop in image quality, yielding mostly blurry reconstructions.
In contrast, DPIR, WCRR, and Score-DiffMBIR retain sharper anatomical structures, and differences among them are less clear. Because the comparison is qualitative and subject to inter-scan differences, no conclusion regarding clinical superiority can be drawn from this example.

\begin{figure*}[!ht]
\centering
    \includegraphics[width=\linewidth]{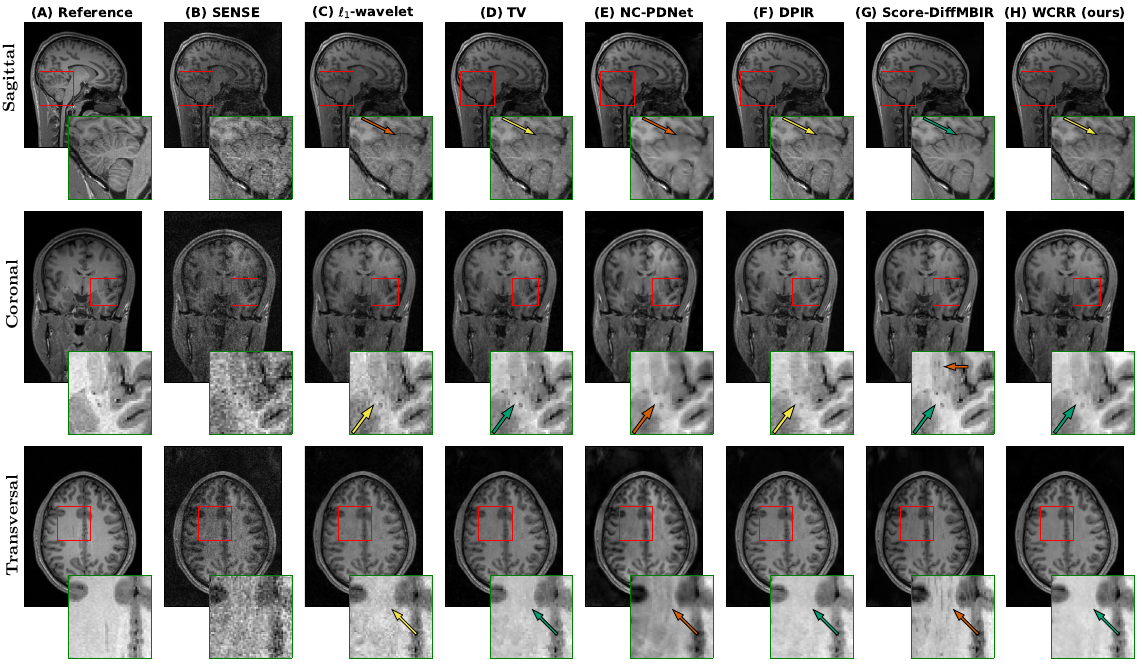}
\caption{\, \capFigEight}
\label{fig:prospective_caipi}
\end{figure*}

\subsection{Interpretation as a sparsity prior}\label{subsec:filters_and_potential}

Figure \ref{fig:filters_and_potential} shows the learned 3D filters of our rotation-invariant WCRR prior alongside the shared potential.
A dynamic visualization of the filters rotating about their $z$ axes is available on the GitHub repository\footnotemark[1] associated with this paper.
Several filters exhibit structures reminiscent of classical sparsifying operators. Filter~7 and, to a lesser extent, Filter~16 resemble the directional finite-difference filters of total variation, with separated positive and negative lobes that respond to local intensity changes.
Filters~3, 10, and 11 show broader oscillatory sign patterns, closer to wavelet-like band-pass filters, while Filter~14 has a localized center-surround structure similar to second-order or blob-detection operators.
The rotation-invariant construction in \eqref{reg} ties together rotated versions of each filter, so the prior favors no particular spatial orientation.
The shared potential resembles a smooth, saturating extension of the \(\ell_1\) penalty: symmetric and minimized at zero, it encourages sparse filter responses while penalizing strong ones less aggressively.
This helps to avoid the blurring of structure during reconstruction.
In summary, WCRR can be seen as a data-driven, rotation-invariant extension of classical sparsity-promoting priors.

\section{Discussion}

The experiments show that the proposed rotation-invariant, sparsity-promoting WCRR regularizer bridges classical variational reconstruction and large learned reconstruction networks.

The comparison with Score-DiffMBIR is particularly relevant, since both methods learn from denoising and can be paired with different acquisition operators at inference time.
Score-DiffMBIR uses a large score network and alternates diffusion-prior denoising with data-consistency updates, whereas WCRR learns a compact regularizer whose gradient directly minimizes the explicit variational objective.
Score-DiffMBIR attains the best numerical scores in our experiments, and WCRR the second best with a far smaller model and comparable reconstruction time.
WCRR should therefore be viewed as a compact, interpretable alternative to diffusion-prior reconstruction, rather than a replacement for diffusion models in terms of peak numerical performance.

The comparison with the end-to-end DRUNet3D and NC-PDNet must be read in light of the available training data.
The used GoLF-SPARKLING acquisition is fully three-dimensional, so its aliasing pattern is not slice-separable.
Thus, training has to be performed on complete volumes.
In the reported experiments, we use the largest available pool—the combined training and validation splits, 67 volumes in total—whereas WCRR used only the original training split of 47 volumes.
In the supplementary material, we report additional experiments measuring performance as the end-to-end baselines' training set varies over 27, 47, and 67 volumes.
These suggest that the end-to-end baselines would improve further with more training volumes.

A practical advantage of WCRR in 3D non-Cartesian MRI is that its training is patch-based and relies solely on denoising, never requiring simulation of the acquisition operator.
The NUFFT operator associated with the trajectory and coil sensitivities enters only through the data-fidelity term at reconstruction time.
Consequently, the same learned prior can be applied to different sampling operators without retraining the acquisition-dependent part of the model. The prospective experiments illustrate this flexibility qualitatively. This operator independence is also a
limitation, however: WCRR does not learn trajectory-specific aliasing removal, and an end-to-end model trained on enough representative data may exploit acquisition-specific structure more effectively.

Several limitations remain.
First, WCRR requires tuning the scalar hyperparameters \(\lambda\) and \(\sigma\), which we performed on the validation set.
A change in scanner, contrast, acceleration, or noise regime may call for retuning these scalars, and larger distribution shifts may require fine-tuning the regularizer.
Second, the training references were generated from the available data by zero-filling the missing k-space region before coil combination.
Fully sampled references with matching contrast and resolution would provide cleaner training targets, whereas the zero-filled references used here could introduce method-specific biases into the learned prior.
Finally, the prospective experiments are limited to a healthy volunteer and therefore contain no targeted pathological findings. Although the zoomed qualitative comparisons allow visual assessment of anatomical detail and reconstruction artifacts, they
cannot determine whether pathological features or diagnostically relevant information are preserved. Moreover, no formal expert reader study was performed. Consequently, the prospective results
should be regarded as technical feasibility examples rather than evidence of clinical benefit or superiority. Establishing clinical relevance will require larger prospective cohorts including pathological cases and task-based or expert-reader evaluation.

\section{Conclusions}
We introduced a rotation-invariant WCRR framework for 3D MRI reconstruction. The method learns a compact image prior from denoising patches and uses it as an explicit regularizer in a variational reconstruction. This preserves the interpretability of classical sparsity-promoting methods while replacing fixed transforms by learned 3D filters and potentials for expressivity.
In retrospective experiments, WCRR achieved reconstruction quality competing with our diffusion-prior baseline while using orders of magnitude fewer trainable parameters. The learned filters and shared potential further support the interpretation of WCRR as a data-driven, rotation-invariant sparsity prior. Prospective GoLF-SPARKLING and CAIPIRINHA experiments illustrate the applicability of the same learned prior under changes in acquisition but do not establish clinical superiority. Clinical assessment will require larger prospective cohorts including pathological cases and expert-reader or task-based evaluation, which will be the subject of future work. Another direction for future work would be replacing the constant vector $\mathbf{1}$ in \eqref{reg} with spatially varying weights, as done for 2D inverse problems in \cite{neumayer2025stability}, which could further improve reconstruction quality.

\begin{figure*}[!ht]
    \centering
    \includegraphics[trim=2 4 1 4, clip, width=1.0\linewidth]{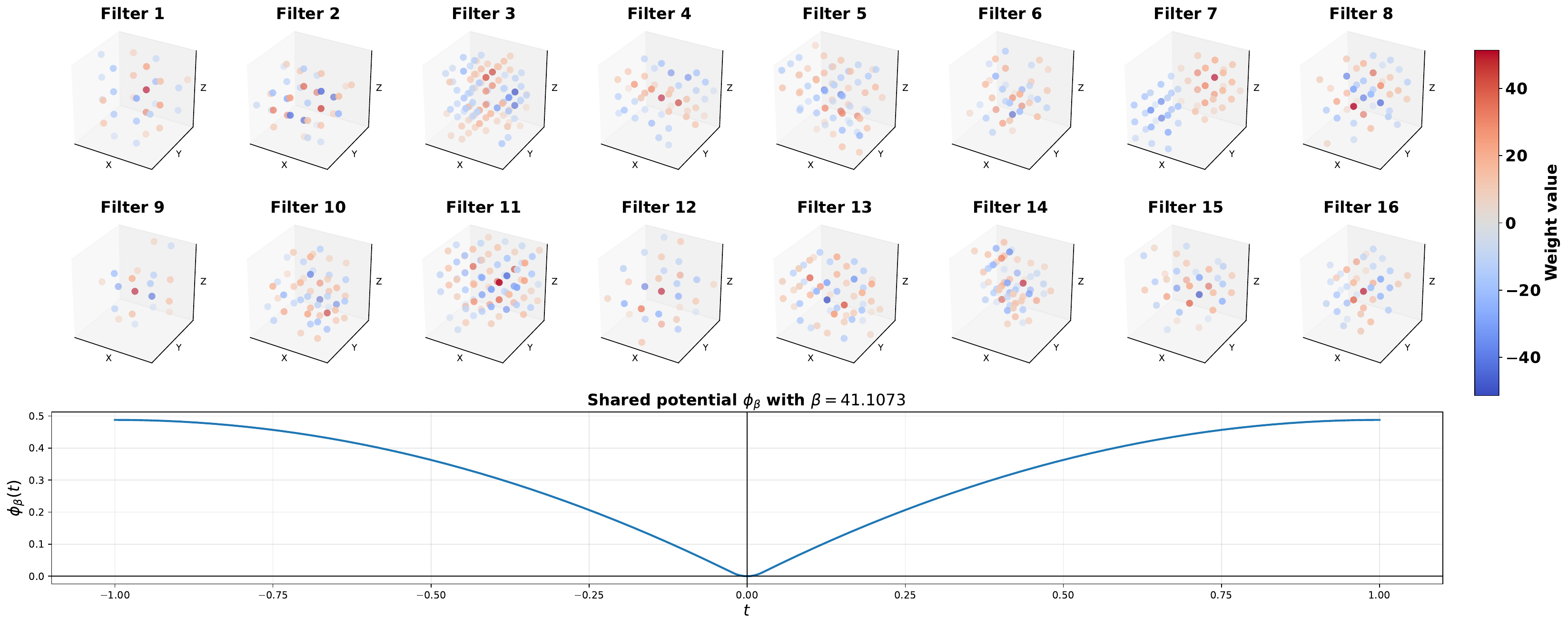}
    \vspace{-0.15cm}
    \caption{\, \capFigSix}
    \label{fig:filters_and_potential}
\end{figure*}

\section*{Acknowledgment}
GSW and SN acknowledge support from the DFG
within the SPP2298 under the Project Number 543939932.
This work was granted access to the CCRT HPC facility under the Grant CCRT2026-tanaasma and CCRT2026-gilirach awarded by the Fundamental Research Division (DRF) of CEA. 
SN wants to thank Matthieu Terris for initiating this project through a discussion at BASP 2025.





\normalsize
\bibliography{references}















\newpage

\onecolumn

\renewcommand{\thesection}{S}
\renewcommand{\theequation}{S\arabic{equation}}
\renewcommand{\thefigure}{S.1.\arabic{figure}}
\renewcommand{\thetable}{S\arabic{table}}

\section*{\centering \large Supplementary material}

\subsection{Sensitivity of the denoising power $\sigma$ and the regularization strength $\lambda$ on the WCRR reconstruction}
As seen in Figure \ref{fig:sigma_sensitivity}, given a fixed $\lambda$, the WCRR reconstruction with a very small $\sigma$ is very noisy.
The noise gradually disappears as we increase $\sigma$.
Then, after the point of almost no noise, the reconstruction just gets blurrier.
Indeed, a regularizer with small $\sigma$ in \eqref{optimization} leads to simply fitting the data.
On the other hand, for a very high $\sigma$, the regularizer leads to aggressive smoothing.

The regularization strength $\lambda$ exhibits the same behavior as shown in Figure \ref{fig:lambda_sensitivity}.
This is expected for similar reasons.
In fact, for a very small $\lambda$, the regularization cost in dominated by the data-fidelity term.
In the opposite regime, more emphasis is put on the regularization cost, whose minimization leads again to aggressive penalization of the high-frequency information, which leads in turn to visible blur.

In practice, fixing $\sigma$ and only optimizing $\lambda$ was sufficient to obtain good reconstructions.

\begin{figure}[!ht]
    \centering
    \includegraphics[trim=2 4 1 4, clip, width=1.0\linewidth]{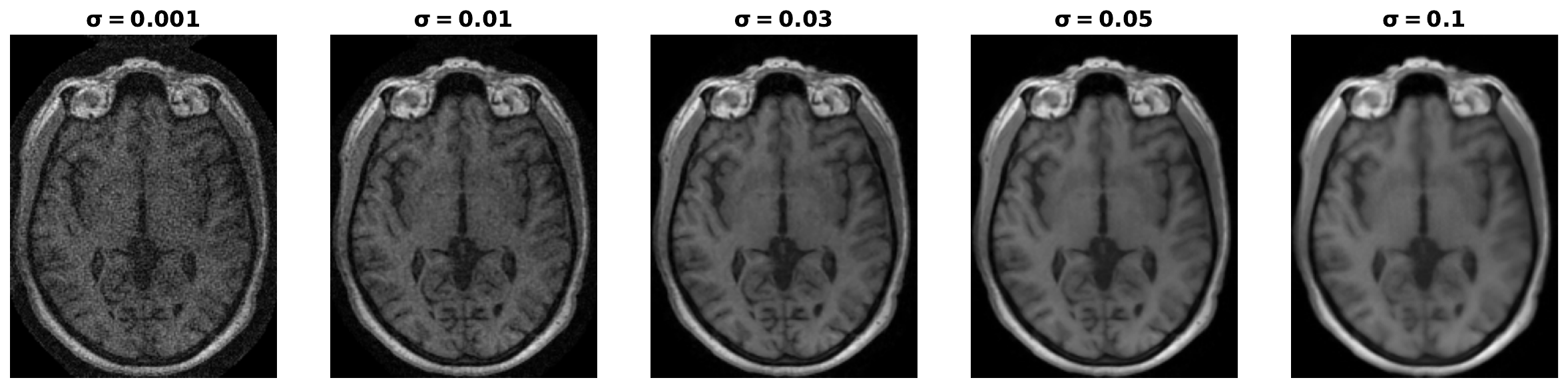}
    \vspace{-0.15cm}
    \caption{\, Sensitivity of the parameter $\sigma$ on a representative WCRR reconstruction (with fixed $\lambda = 0.01$). Only a transversal slice is displayed.}
    \label{fig:sigma_sensitivity}
\end{figure}

\begin{figure}[!ht]
    \centering
    \includegraphics[trim=2 4 1 4, clip, width=1.0\linewidth]{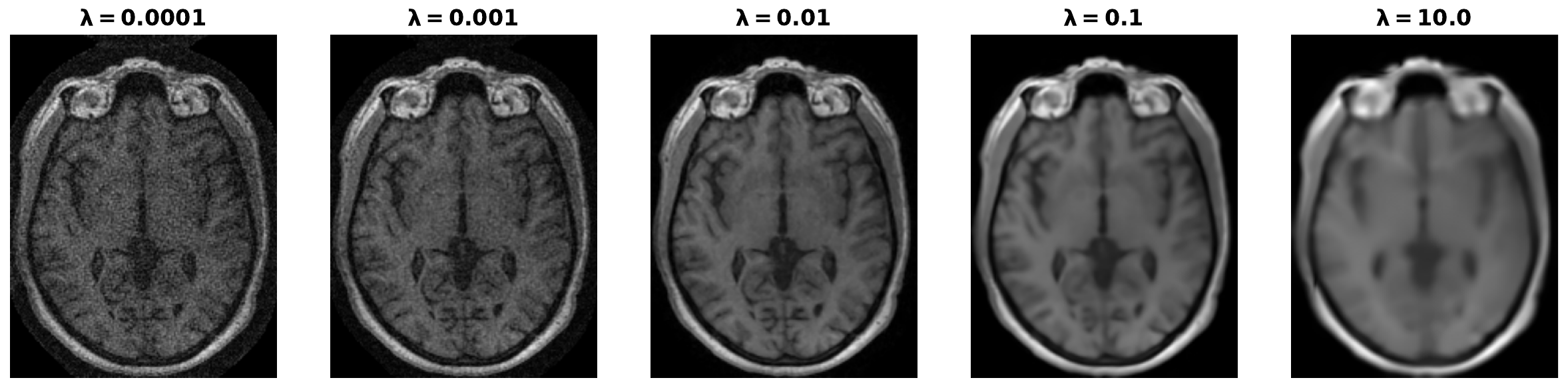}
    \vspace{-0.15cm}
    \caption{\, Sensitivity of the parameter $\lambda$ on a representative WCRR reconstruction (with fixed $\sigma = 0.03$). Only a transversal slice is displayed.}
    \label{fig:lambda_sensitivity}
\end{figure}

\subsection{Data-efficiency of end-to-end methods and WCRR}

\renewcommand{\thefigure}{S.2.\arabic{figure}}

Figure \ref{fig:masked_metrics_vs_training_size} shows that the quantitative metrics (both masked PSNR and masked SSIM) of the two end-to-end methods (DRUNet3D and NC-PDNet) grow almost linearly with the training set size, while WCRR starts to saturate in terms of performance.
This suggests that, given enough data, the performance of the two end-to-end methods might improve beyond WCRR's.
At the same time, it confirms that WCRR does not require much data to be trained and is therefore more data-efficient than its end-to-end counterparts.

\begin{figure}[!ht]
    \centering
    \includegraphics[trim=2 4 1 4, clip, width=1.0\linewidth]{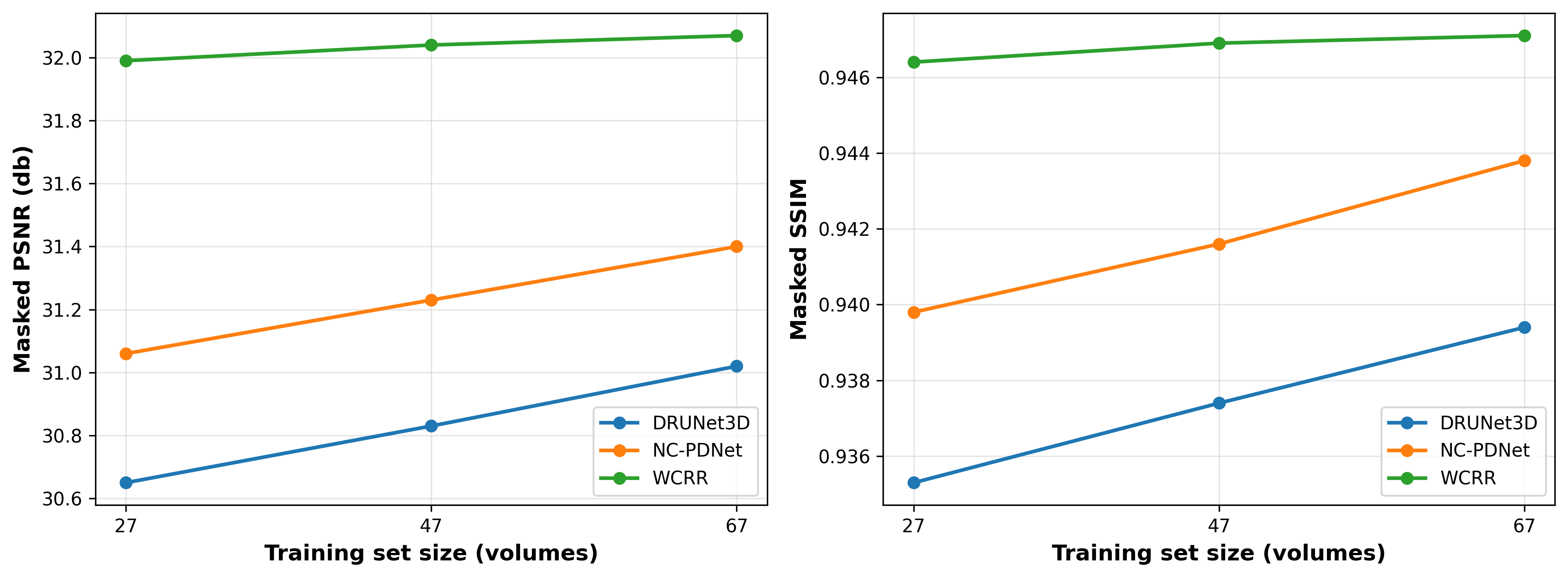}
    \vspace{-0.15cm}
    \caption{\, Quantitative metrics (averaged across the retrospective 12-coil testing set) vs training set size, for the end-to-end methods (DRUNet3D and NC-PDNet) and WCRR.}
    \label{fig:masked_metrics_vs_training_size}
\end{figure}

\subsection{Complementary training details for DRUNet3D, NC-PDNet \& Score-DiffMBIR}

Unlike WCRR which was trained using only 47 volumes, the following three methods were trained using 67 volumes (the combination of the 47 training and 20 validation volumes of the original split).
The results reported in the paper correspond to the above-mentioned setting.

\paragraph{DRUNet3D}
Here, the training was performed with batch size 1 and gradient accumulation over 4 iterations, corresponding to an effective batch size of 4. The model was optimized for 500 epochs using the Adam optimizer with learning rate $\num{1e-4}$ and gradient clipping with maximum norm 1.

\paragraph{NC-PDNet}

The input to the network consists of the reconstruction via SENSE.
The training was performed with a batch size of 1 for up to 180 epochs using the AdamW optimizer with weight decay $\num{1e-3}$.
The learning rate was initialized as $5 \times 10^{-4}$ and decayed every epoch according to a cosine annealing schedule to \num{1e-5}.

\paragraph{Score-DiffMBIR}

Here, the \emph{NCSN++} variance-exploding score model was trained using $\approx 11\,000$ sagittal slices extracted from the 67 complex-valued volumes.
The first and last five slices of each volume were discarded.
Data augmentation included random left-right flipping, small translations of up to 8 pixels, and global phase perturbations.
We trained with \(\sigma \in [0.001,0.1]\) for \(200{,}000\) steps with batch size 8 using the Adam optimizer with learning rate \(10^{-4}\).
A linear learning-rate warm-up was applied during the first \(2.5\%\) of training, starting from \(0.1\) times the base learning rate.
The training was stopped after \(35\,000\) steps, a time after which the denoising performance started to stagnate on a held-out validation set.

\end{document}